\documentclass[twoside,leqno,twocolumn]{article}

\usepackage[letterpaper]{geometry}

\usepackage{siamproceedings}

\usepackage[T1]{fontenc}
\usepackage{amsfonts}
\usepackage{epstopdf}
\usepackage{enumitem}
\ifpdf
  \DeclareGraphicsExtensions{.eps,.pdf,.png,.jpg}
\else
  \DeclareGraphicsExtensions{.eps}
\fi

\newsiamremark{remark}{Remark}
\newsiamremark{hypothesis}{Hypothesis}
\crefname{hypothesis}{Hypothesis}{Hypotheses}
\newsiamthm{claim}{Claim}

\usepackage{amsmath,amssymb,amsfonts}
\usepackage{graphicx}
\usepackage{booktabs}
\usepackage{xcolor}
\usepackage{array}
\usepackage{textcomp}
\usepackage{hyperref}
\usepackage{url}
\usepackage{multirow}
\usepackage{algpseudocode}

\usepackage{amsopn}

\begin{document}

\title{\Large Spatial Entropy based Partitioning for Spatiotemporal Graph Unlearning}

\author{Qiming Guo \footnotemark[3] 
\and
Wenbo Sun
\thanks{Delft University of Technology, Netherlands
(\email{w.sun-2@tudelft.nl}).}
\and
Ye Wang
\thanks{Biogen, Cambridge, MA, USA
(\email{ye.wang@biogen.com}).}
\and
Wenlu Wang 
\thanks{Texas A\&M University - Corpus Christi, Corpus Christi, TX, USA (\email{qguo2@islander.tamucc.edu}, \email{wenlu.wang@tamucc.edu}).}
}

\date{}

\maketitle

\fancyfoot[R]{\scriptsize{Accepted at the SIAM International Conference on Data Mining (SDM 2026)}}

\begin{abstract} \small\baselineskip=9pt

Spatiotemporal graphs underpin applications such as traffic forecasting, weather forecasting, and healthcare monitoring. Privacy regulations such as the GDPR and the CCPA require the complete removal of unauthorized data from trained models, but achieving this on a spatiotemporal graph is difficult: because information propagates globally through both spatial and temporal message passing, fully erasing a node's influence forces costly full-graph retraining. ST-graph unlearning requires both exactness and efficiency. We propose IsleNet, which uses spatial-entropy-guided partitioning to create balanced, locally coherent subgraphs and reconnects them with lightweight virtual edges. Upon an unlearning request, only the affected subgraph encoder and virtual-edge layer are retrained, ensuring exact removal with low cost. Experiments on four real-world benchmarks show that IsleNet attains up to 94\% of full-graph accuracy while reducing unlearning time by up to an order of magnitude. Our code is publicly available at \url{https://github.com/wenlu-lab/STGraphUnlearning}.
\end{abstract}


\section{Introduction.}
Privacy regulations such as the GDPR~\cite{eu_gdpr_2016} and the CCPA~\cite{ca_ccpa_2018} grant individuals the right to have their data erased from trained machine learning models. For spatiotemporal graph neural networks (ST-GNNs), which underpin applications from traffic forecasting to healthcare monitoring, honoring such requests is particularly difficult. To fully erase the influence of a requested node, the standard recourse is to retrain the entire ST-graph from scratch, because every node's features propagate globally through both spatial and temporal message passing, and any residual path can leak the deleted data back into the model. But full retraining has two severe consequences (Figure~\ref{figure1}): it breaks the information-propagation walks that hold the ST-graph together and can even structurally isolate nodes that were never requested for deletion, and it costs as much time and compute as training the original model. Spatiotemporal graph unlearning therefore requires simultaneously meeting two goals: exactness, meaning the deleted data's influence must be completely removed, and efficiency, meaning the process must be fast without sacrificing post-unlearning accuracy.

A natural remedy is to partition the graph into independent subgraphs, train each in isolation, and retrain only affected partitions upon deletion. Spectral clustering~\cite{Yu2018STGCN} groups topologically tight nodes into locally coherent subgraphs (Figure~\ref{figure2}), letting each encoder exploit dense local message passing. The remaining challenge is balance: if one subgraph collapses to a trivially small cluster its encoder degenerates. We use spatial entropy as the verification criterion, since a well-balanced partition produces subgraphs of comparable entropy, and our normalized partition entropy $\bar{H}_{\mathcal{P}}$ triggers redistribution when balance fails.

We propose IsleNet (Island Network for Unlearning Spatiotemporal Graphs). In Stage~1, the ST-graph is partitioned into $M$ entropy-verified subgraphs using spatial-entropy-guided partitioning, each trained by an independent encoder with gradient-isolated optimizers. In Stage~2, encoders are frozen and lightweight virtual edges are optimized to compensate for severed cross-subgraph dependencies. Because the virtual edges are structurally decoupled from the encoders, an unlearning request only retrains the affected subgraph encoder and the virtual-edge layer, achieving exact localized deletion. This design simultaneously satisfies both exactness and efficiency on spatiotemporal graphs. Our main contributions are:

% \begin{figure*}[htbp]
\begin{figure*}[h]
    \centering
    \includegraphics[width=.99\linewidth]{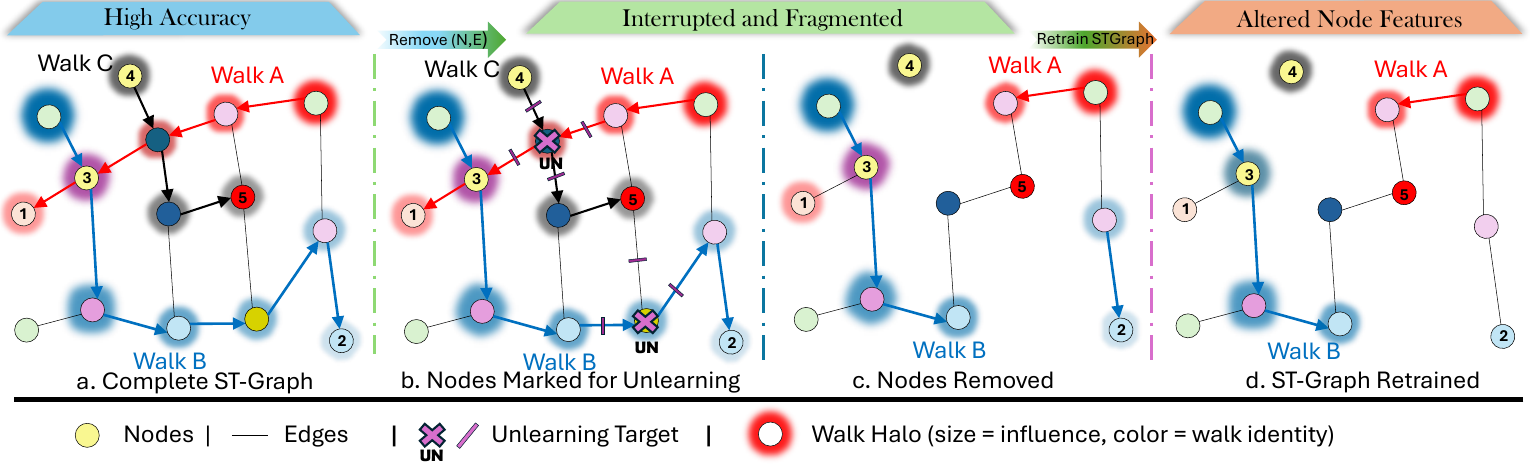}
\caption{\textbf{The challenge of unlearning on spatiotemporal graphs.} (a) In a complete ST-graph, walks propagate node influence globally (e.g., $v_4$ reaches $v_5$ via Walk~C; $v_3$ lies on Walks A and B). (b) Target nodes are marked for unlearning (\textbf{UN}), blocking every walk through them. (c) After removal, Walk~A no longer reaches $v_1$ or $v_3$, Walk~B no longer reaches $v_2$, and uninvolved $v_4$ loses all outgoing paths, a structural cascade. (d) Retraining the fragmented graph alters the remaining node features ($v_1, v_2, v_3, v_5$) and fades the walk halos that encode global context, while costing as much as full training. ST-graph unlearning must therefore be \emph{exact} (deleted influence fully removed) and \emph{efficient} (fast without sacrificing post-unlearning accuracy).}

% \caption{\textbf{The challenge of unlearning on spatiotemporal graphs.} To fully erase the influence of target nodes and edges, conventional unlearning methods must retrain the entire ST-graph from scratch, which causes three fundamental problems. First, information flow is broken: (a) in a complete ST-graph, walks propagate node influence globally (e.g., node $v_4$ transmits information to $v_5$ via Walk~C; $v_3$ lies at the intersection of Walks A and B). (b) Once target nodes are marked for removal (\textbf{UN}) and erased (c), every walk passing through them is severed. Walk~A no longer reaches $v_1$ or $v_3$, Walk~B no longer reaches $v_2$, and Walk~C can no longer carry $v_4$'s signal to $v_5$. Second, uninvolved nodes become isolated: $v_4$ itself was never in the deletion set, yet loses all its outgoing paths, a cascade effect that damages the ST-graph structurally. Third, full retraining is costly: (d) retraining the fragmented graph distorts the remaining features ($v_1, v_2, v_3, v_5$) and strips away global context (halos largely vanish), while consuming large amounts of time and compute. These observations define the two core requirements for ST-graph unlearning: (i) exactness, meaning the influence of the deleted nodes and edges must be fully removed, and (ii) efficiency, meaning unlearning must be fast without sacrificing post-unlearning accuracy. IsleNet is designed to meet both.}
    \label{figure1}\vspace{-15pt}
\end{figure*}

\noindent\textbf{1. Spatial entropy as a partition balance criterion.} We propose a new concept ``spatial entropy'' as a principled measure for verifying that spectral partitioning produces balanced, non-degenerate subgraphs, each preserving dense local connectivity so that its encoder can fully exploit local spatiotemporal correlations.

\noindent\textbf{2. Virtual edges for decoupled global integration.} We design a two-stage training protocol in which independently trained subgraph encoders are connected through lightweight virtual edges. The strict gradient isolation between stages guarantees exact unlearning: retraining is confined to the affected subgraph and the virtual-edge parameters.
    
\noindent\textbf{3. Comprehensive empirical validation.} Across four real-world datasets spanning 23 to 3{,}220 nodes and four ST-GNN backbones, IsleNet attains up to 94\% of full-graph accuracy while reducing unlearning time by up to 12$\times$, outperforming five baselines.

\section{Related Work.} \label{sec:related} 

A Spatiotemporal Graph $\mathcal{G} = (\mathcal{V}, \mathcal{E}, \mathbf{X})$ has $N$ nodes with adjacency matrix $\mathbf{A} \in \mathbb{R}^{N \times N}$ and features $\mathbf{X} \in \mathbb{R}^{T \times N \times F}$ spanning $T$ timesteps. ST-GNNs such as STGCN~\cite{Yu2018STGCN}, ST-GAT~\cite{Velickovic2018GAT}, and DCRNN~\cite{Li2018DCRNN} couple spatial message passing with recurrent or convolutional temporal operators, producing per-node forecasts. This coupling is precisely what makes unlearning expensive: each node's features propagate through both dimensions, so a single deletion can require full retraining.

General-purpose unlearning falls into two camps. Exact methods retrain affected subsets of the model: SISA~\cite{bourtoule2021machine} partitions data into independent shards so only affected shards are retrained; Ginart et al.~\cite{ginart2019making} derive deletion-efficient $k$-means variants. Approximate methods modify trained parameters directly: certified removal~\cite{guo2020certified} uses Newton updates with differential-privacy noise, and weight scrubbing~\cite{golatkar2020eternal} perturbs weights along Fisher-information directions. Exact methods offer formal guarantees; approximate methods are faster but lack provable deletion. GraphEraser~\cite{chen2022grapheraser} pioneered graph unlearning through balanced partitioning (BLPA/BEKM) with learned aggregation; GUIDE~\cite{wang2023inductive} extends this with fairness-aware dynamic partitioning. GNNDelete~\cite{cheng2023gnndelete} avoids partitioning by learning layer-wise deletion operators, while PROJECTOR~\cite{cong2023projector} provides exact deletion for linear GNNs via orthogonal projection. These methods target static graph node classification; none address spatiotemporal forecasting, where temporal message passing amplifies residual influence. For spatiotemporal settings, STEPS~\cite{guo2025efficient} applies spectral partitioning with weighted aggregation to ST-GNN forecasting, handling up to 15\% node removal with order-of-magnitude speedups, but lacks formal deletion proofs and requires preset partition scales. Graph Revoke~\cite{zhang2025dynamic} unlearns dynamic edges in continuous-time graphs via ML-Mixer gradient transformation, achieving 7.23$\times$ speedups on link prediction, but offers no theoretical deletion guarantees and may face scalability bottlenecks.

Information-theoretic measures have a long history in graph analysis~\cite{dehmer2011history,korner1973coding,li2016structural} and in partition quality assessment~\cite{rosvall2008maps,dhillon2004kernel}. Spatial entropy specifically has been used in geographic information science to measure the evenness of point distributions~\cite{batty1974spatial,leibovici2009defining}. To our knowledge, no prior work applies entropy-based criteria to verify partition quality in spatiotemporal graph unlearning.

\section{Method.}
\label{sec:method}

IsleNet partitions the ST-graph into locally coherent subgraphs that preserve dense internal connectivity, with spatial entropy verifying that no subgraph degenerates. Lightweight virtual edges, structurally decoupled from the encoders, compensate for severed long-range dependencies. A two-stage protocol first trains each encoder in isolation, then freezes them and trains only the virtual edges, confining each node's influence to one subgraph so that unlearning retrains only the affected encoder and the virtual-edge layer.

\begin{figure}[t]
    \centering
    \includegraphics[width=.8\linewidth]{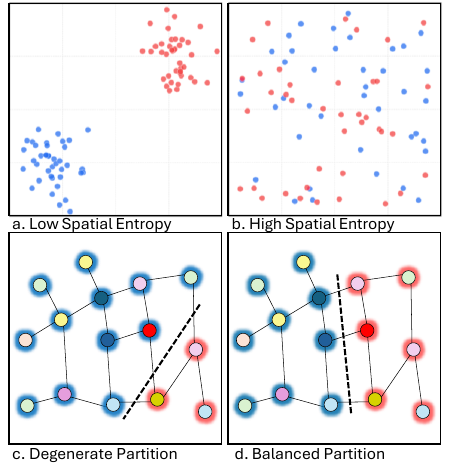}
\caption{\textbf{Spatial entropy in IsleNet.}
\textbf{Top:} Spatial entropy $H_s$ (Eq.~\ref{eq:spatial_entropy}) quantifies how a subgraph's nodes are distributed across space. (a)~Low: nodes concentrate in narrow regions. (b)~High: nodes spread uniformly.
\textbf{Bottom:} IsleNet uses spatial entropy to verify partition quality: blue and red halos mark the two subgraphs, and the dashed line separates them. (c)~A degenerate partition: one subgraph shrinks to a small corner, so IsleNet detects the imbalance and redistributes nodes. (d)~A balanced partition: both subgraphs cover comparable spatial extents.}
    \label{figure2}\vspace{-15pt}
\end{figure}

\begin{figure*}[t]
    \centering
    \includegraphics[width=0.99\linewidth]{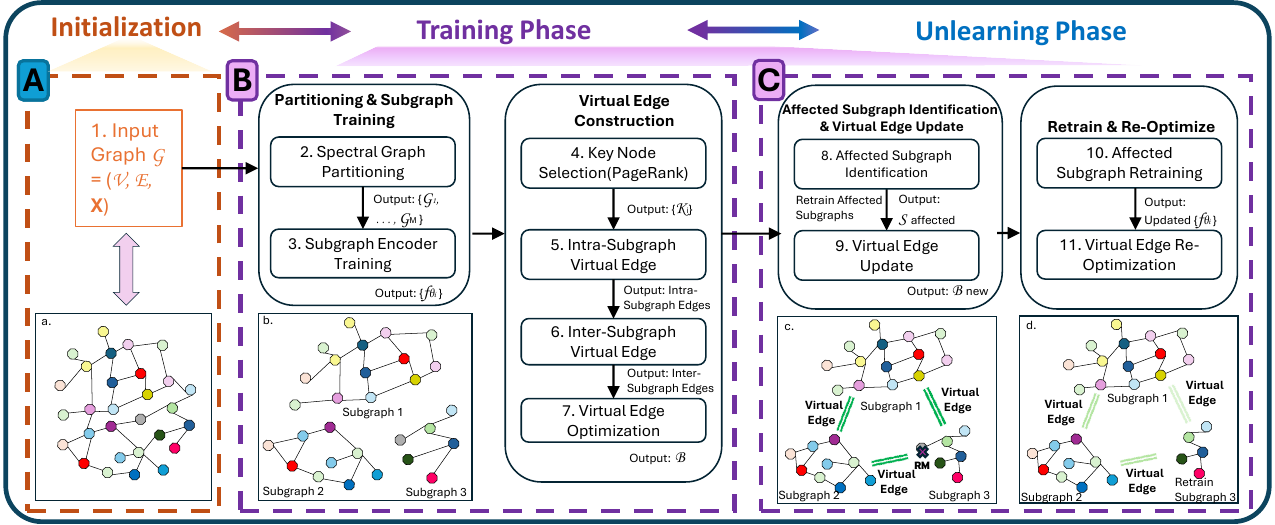}
\caption{\textbf{Framework of IsleNet.} 
(a) The original graph. 
(b) Split into smaller subgraphs, each with its own encoder. 
(c) Virtual edges (green) reconnect subgraphs; on a deletion request, edges touching the target are marked for removal (RM). 
(d) Only the affected subgraph is retrained and virtual edges are rebuilt; other subgraphs stay untouched.
\textbf{How it works:} Spatial entropy guides the split so no subgraph degenerates. Each encoder trains in isolation, and virtual edges stitch subgraphs back together. To unlearn, IsleNet retrains only the affected encoder and rebuilds the virtual edges from scratch. Untouched subgraphs keep their weights, making unlearning fast and exact.}
    \label{figure3}\vspace{-15pt}
\end{figure*}

\subsection{Problem Formulation. }

Consider a spatiotemporal graph $\mathcal{G} = (\mathcal{V}, \mathcal{E}, \mathbf{X})$, where $\mathcal{V}$ is a set of $N$ nodes, $\mathcal{E}$ encodes edges in a binary adjacency matrix $\mathbf{A} \in \{0,1\}^{N \times N}$, and $\mathbf{X} \in \mathbb{R}^{T \times N \times F}$ contains temporal features of length $T$ with $F$-dimensional features per node. A GNN model $f_{\theta}: (\mathbf{X}, \mathbf{A}) \mapsto \mathbf{Y}$ predicts future node states $\mathbf{Y} \in \mathbb{R}^{N \times H}$ over a horizon $H$. The unlearning task is to remove a subset of nodes $\mathcal{U}_N \subset \mathcal{V}$ and their edges $\mathcal{U}_E = \{(u,v) \in \mathcal{E} : u \in \mathcal{U}_N \text{ or } v \in \mathcal{U}_N\}$, producing a new model $f_{\theta'}$ that approximates the performance of a model retrained from scratch on the remaining graph $\mathcal{G}' = (\mathcal{V} \setminus \mathcal{U}_N, \mathcal{E} \setminus \mathcal{U}_E, \mathbf{X}')$.

Formally, given a loss function $\mathcal{L}$ (e.g., L1 loss), the unlearning objective is:
\begin{equation}
\theta' \approx \arg\min_{\theta} \mathbb{E}_{(\mathbf{X}',\mathbf{Y}')} \mathcal{L}(f_{\theta}(\mathbf{X}', \mathbf{A}'), \mathbf{Y}'),
\label{eq:unlearn_objective}
\end{equation}

where $\mathbf{A}'$ is the adjacency matrix of $\mathcal{G}'$, and $\mathbf{X}', \mathbf{Y}'$ exclude data associated with $\mathcal{U}_N$. Following the SISA framework~\cite{bourtoule2021machine}, we say that unlearning is exact if the post-unlearning model $f_{\theta'}$ is statistically indistinguishable from a model retrained from scratch under the same partition and training protocol on $\mathcal{G}'$. IsleNet satisfies this condition because gradient isolation (Eq.~\ref{eq:gradient_isolation}) ensures that no encoder parameter outside the affected subgraph was ever influenced by the deleted nodes, and the virtual-edge layer is re-initialized and re-trained on the remaining data.

\subsection{IsleNet Architecture. }

IsleNet decomposes the graph into $M$ independent subgraphs and integrates them via a virtual edge mechanism to balance local isolation and global information flow. The architecture comprises three components: spatial-entropy-guided partitioning, subgraph-independent spatiotemporal encoding, and virtual edges.

\subsubsection{Spatial-Entropy-Guided Partitioning. }
\label{sec:partition}

We partition $\mathcal{G}$ into $M$ disjoint subgraphs $\{\mathcal{G}_1, \ldots, \mathcal{G}_M\}$, where $\mathcal{G}_i = (\mathcal{V}_i, \mathcal{E}_i, \mathbf{X}_i)$, $\mathcal{V}_i \cap \mathcal{V}_j = \emptyset$ for $i \neq j$, and $\bigcup_i \mathcal{V}_i = \mathcal{V}$. A key insight of IsleNet is that locally coherent subgraphs with dense internal connectivity allow each encoder to fully exploit spatiotemporal correlations through GCN message passing. However, the partition must also be balanced: no subgraph should degenerate to a trivially small or collapsed cluster. We use spatial entropy to formalize and verify this balance property.

We define \textbf{Spatial Entropy} by discretizing the spatial domain of $\mathcal{G}$ into $R$ non-overlapping regions (e.g., grid cells based on node coordinates). For any node subset $\mathcal{V}' \subseteq \mathcal{V}$, let $p_r(\mathcal{V}') = |\{v \in \mathcal{V}' : \mathbf{s}_v \in \text{region } r\}| / |\mathcal{V}'|$ denote the fraction of nodes falling in region $r$. The spatial entropy of $\mathcal{V}'$ is defined as:
\begin{equation}
H_s(\mathcal{V}') = -\sum_{r=1}^{R} p_r(\mathcal{V}') \log p_r(\mathcal{V}'),
\label{eq:spatial_entropy}
\end{equation}
where $0 \log 0 \triangleq 0$. The maximum entropy $H_s^{\max} = \log R$ is attained when nodes are uniformly distributed across all regions.

For a partition $\mathcal{P} = \{\mathcal{V}_1, \ldots, \mathcal{V}_M\}$, we define the \textbf{partitioning entropy} as the minimum spatial entropy across all subgraphs, together with its normalized form:
\begin{align}
H_{\mathcal{P}} &= \min_{i \in \{1,\ldots,M\}} H_s(\mathcal{V}_i),
\label{eq:partition_entropy} \\
\bar{H}_{\mathcal{P}} &= \frac{H_{\mathcal{P}}}{H_s(\mathcal{V})}.
\label{eq:normalized_entropy}
\end{align}
$\bar{H}_{\mathcal{P}}$ ranges from 0 (a degenerate partition where one subgraph has collapsed) to 1 (all subgraphs have comparable spatial extent). We accept a partition when $\bar{H}_{\mathcal{P}} \geq 1-\epsilon$; otherwise balanced redistribution is triggered (detailed below).

%\paragraph{Partitioning Objectives. }
IsleNet seeks \textbf{a partition that minimizes inter-subgraph edge cuts} (preserving local connectivity) subject to balance constraints verified by partition entropy:
\begin{equation}
\min_{\mathcal{P}} \sum_{i=1}^M \sum_{\substack{u \in \mathcal{V}_i \\ v \notin \mathcal{V}_i}} \mathbf{A}_{uv} \quad \text{s.t.} \quad \left| |\mathcal{V}_i| - \frac{N}{M} \right| \leq \delta, \;\; \bar{H}_{\mathcal{P}} \geq 1-\epsilon,
\label{eq:partition}
\end{equation}
where $\delta$ controls size tolerance and $\bar{H}_{\mathcal{P}}$ (defined in Eq.~\ref{eq:normalized_entropy}) is the normalized partition entropy. Algorithmically, we first run spectral clustering on the graph Laplacian $\mathbf{L} = \mathbf{D} - \mathbf{A}$ with balance constraint $|\mathcal{V}_i| \approx N/M$, producing an initial partition $\mathcal{P}^{(0)}$. We then compute $\bar{H}_{\mathcal{P}^{(0)}}$ over a discretized grid on the node coordinates. If $\bar{H}_{\mathcal{P}^{(0)}} \geq 1-\epsilon$ the partition is accepted; otherwise, a balanced redistribution procedure (below) is invoked to produce a spatially non-degenerate partition. We use $\epsilon = 0.1$ throughout.

When $\bar{H}_{\mathcal{P}^{(0)}} < 1-\epsilon$, we iteratively rebalance: at each step, we identify the subgraph $\mathcal{V}_{i^*}$ with the lowest spatial entropy, select the node farthest (in Euclidean distance) from its centroid, and reassign it to the subgraph with highest entropy among those with $|\mathcal{V}_j| < N/M + \delta$; then recompute $\bar{H}_{\mathcal{P}}$. The procedure terminates when $\bar{H}_{\mathcal{P}} \geq 1-\epsilon$ or after three iterations; on failure we fall back to a size-balanced random partition. To quantify \textbf{balance across subgraphs} beyond the minimum criterion, we also monitor the entropy variance $\sigma^2_H = \tfrac{1}{M} \sum_{i=1}^M ( H_s(\mathcal{V}_i) - \bar{H}_s )^2$, where $\bar{H}_s = \tfrac{1}{M}\sum_i H_s(\mathcal{V}_i)$. A partition is \emph{spatially adequate} if both $\bar{H}_{\mathcal{P}} \geq 1-\epsilon$ (non-degeneracy) and $\sigma^2_H \leq \tau$ (low imbalance). IsleNet's spectral clustering with balance constraints satisfies both on all benchmarks with $\epsilon \leq 0.1$ and $\tau \leq 0.01$.

\begin{proposition}[Representational Fidelity]
\label{prop:fidelity}
If a partition $\mathcal{P}$ satisfies $H_s(\mathcal{V}_i) \geq H_s(\mathcal{V}) - \epsilon$ for all $i$, then each subgraph $\mathcal{G}_i$ occupies at least a fraction $(1-\epsilon)$ of the spatial regions of $\mathcal{G}$, so no subgraph is spatially degenerate and every encoder receives a non-trivial portion of the graph's extent.
\end{proposition}

\begin{proof}
Let $R_i = |\{r : p_r(\mathcal{V}_i) > 0\}|$. Since $H_s(\mathcal{V}_i) \leq \log R_i$, we have $R_i \geq e^{H_s(\mathcal{V}_i)} \geq e^{H_s(\mathcal{V}) - \epsilon}$. When $\mathcal{V}$ approximately uniformly covers $R^+$ occupied regions, $H_s(\mathcal{V}) \approx \log R^+$, so $R_i / R^+ \geq e^{-\epsilon} \geq 1-\epsilon$.
\end{proof}

\begin{theorem}[Post-Unlearning Entropy Bound]
\label{thm:unlearn_fidelity}
Under a spatially adequate partition $\mathcal{P}$, removing a fraction $\gamma_i = |\mathcal{U}_N \cap \mathcal{V}_i|/|\mathcal{V}_i|$ of nodes from $\mathcal{V}_i$ yields
$H_s(\mathcal{V}_i \setminus \mathcal{U}_N) \geq H_s(\mathcal{V}_i) - \gamma_i \log R$,
so the retrained subgraph preserves enough spatial coverage for small-fraction unlearning.
\end{theorem}

\begin{proof}
Removing $\gamma_i |\mathcal{V}_i|$ nodes reduces entropy by at most $\gamma_i \log R$ (the information content of uniformly-distributed removed nodes). Combined with $H_s(\mathcal{V}_i) \geq (1-\epsilon) H_s(\mathcal{V})$, the post-unlearning entropy remains close to $H_s(\mathcal{V})$ for small $\gamma_i$.
\end{proof}

\subsubsection{Subgraph-Independent Spatiotemporal Encoding. }

Each subgraph $\mathcal{G}_i$ is processed by a dedicated encoder $f_{\theta_i}$ (e.g., STGCN), taking features $\mathbf{X}_i \in \mathbb{R}^{B \times T \times |\mathcal{V}_i| \times F}$ and edge index $\mathbf{E}_i \subseteq \mathcal{E}_i$, producing embeddings $\mathbf{H}_i \in \mathbb{R}^{B \times |\mathcal{V}_i| \times H}$:
\begin{equation}
\mathbf{H}_i = f_{\theta_i}(\mathbf{X}_i, \mathbf{E}_i).
\label{eq:encoder}
\end{equation}
The encoder applies graph convolutions for spatial modeling and recurrent units for temporal modeling, outputting per-node predictions. Independence is ensured by gradient isolation:
\begin{equation}
\frac{\partial \mathcal{L}_i}{\partial \theta_j} = 0 \quad \forall i \neq j,
\label{eq:gradient_isolation}
\end{equation}
where $\mathcal{L}_i$ is the loss on subgraph $i$. Isolation is enforced mechanically: each encoder $f_{\theta_i}$ has its own Adam optimizer operating on a disjoint parameter group, and $\mathcal{L}_i$ is computed only from $(\mathbf{X}_i, \mathbf{E}_i, \mathbf{y}_i)$, so the computation graph for encoder $j \neq i$ is never instantiated during a step that updates encoder $i$. Consequently, no parameter in $\theta_j$ can depend on any node in $\mathcal{V}_i$, which is the prerequisite for exact localized unlearning.

\paragraph{Gradient Isolation Bound. } For a deletion fraction $\gamma = |\mathcal{U}_N|/N$, the post-unlearning loss satisfies $\mathcal{L}_{\text{unlearn}} \leq \mathcal{L}_{\text{orig}} + O(\gamma \sum_{i \in \mathcal{S}_{\text{affected}}} \mathcal{L}_i)$, ensuring minimal error increase since $|\mathcal{S}_{\text{affected}}| \leq M$.

\subsubsection{Virtual Edge Construction. }

To restore global connectivity lost during partitioning, we introduce virtual edges comprising intra-subgraph and inter-subgraph connections, which aggregate information via attention-based embeddings. In principle, virtual edges can be parameterized by any network designed for information summarization; in this work, we adopt a multilayer perceptron. 

%\paragraph{Endpoints of Virtual Edges. } 
For each subgraph $\mathcal{G}_i$, we select the top $k_i = \max(2, \lceil 0.1 \cdot |\mathcal{V}_i| \rceil)$ structurally influential nodes $\mathcal{K}_i$ using PageRank centrality~\cite{page1999pagerank}. The PageRank vector $\mathbf{p} \in \mathbb{R}^{|\mathcal{V}_i|}$ is computed via the standard power iteration with damping factor $\alpha$, and the key node set is:
\begin{equation}
\mathcal{K}_i = \text{TopK}(\mathbf{p}, k_i).
\label{eq:pagerank_select}
\end{equation}
These key nodes serve as \textbf{endpoints for the virtual edges} that restore cross-subgraph information flow.

%\paragraph{Virtual Edge Construction. } 
Each virtual edge aggregates embeddings from a set of key nodes $\mathcal{K}_{\text{src}}$ via scaled dot-product attention followed by an MLP. Let $\mathbf{H}_{\mathcal{K}} \in \mathbb{R}^{B \times n \times D}$ collect the embeddings of $n$ selected key nodes. The virtual-edge feature is computed as:
\begin{align}
\mathbf{W} &= \text{softmax}\!\left(\frac{\mathbf{H}_{\mathcal{K}} \mathbf{H}_{\mathcal{K}}^\top}{D}\right), \label{eq:attention_w} \\
\mathbf{h}_{b} &= \text{MLP}(\text{flatten}(\mathbf{W} \mathbf{H}_{\mathcal{K}})), \label{eq:attention_mlp}
\end{align}
where $\text{MLP}: \mathbb{R}^{nD} \rightarrow \mathbb{R}^{D_b}$ is a two-layer network with ReLU activation. We instantiate two types of virtual edges: intra-subgraph edges that connect the top-2 key nodes within each subgraph ($n{=}2$, yielding $M$ edges), and inter-subgraph edges that connect the top-3 key nodes from each pair of subgraphs ($n{=}6$, yielding $\binom{M}{2}$ edges). The total number of virtual edges is $M + \binom{M}{2}$. The top-2/top-3 choice is a parameter-efficiency trade-off: the $\binom{M}{2}$ inter-subgraph edges dominate the parameter budget, so a larger receptive field there (six key nodes) is more impactful than widening the $M$ intra-subgraph edges.

Virtual-edge features are aggregated residually:
\begin{equation}
\mathbf{H}^{\text{enhanced}} = \mathbf{H} + \beta \sum_{b \in \mathcal{B}} \mathbf{P}_b \text{Proj}(\mathbf{h}_b),
\label{eq:ve_enhance}
\end{equation}
where $\mathbf{H} = [\mathbf{H}_1; \ldots; \mathbf{H}_M] \in \mathbb{R}^{B \times N \times D}$ concatenates subgraph embeddings, $\text{Proj}: \mathbb{R}^{D_b} \rightarrow \mathbb{R}^{D}$ is a linear layer, $\beta > 0$ is a residual scaling factor tuned on a validation set to balance local and global contributions (we use $\beta{=}0.1$ throughout), and $\mathbf{P}_b \in \{0,1\}^{N \times 1}$ maps virtual edge $b$ to its connected nodes. This approximates $\mathbf{A}$ with a low-rank structure, ensuring efficient global information flow.

The attention weights $\mathbf{W}$ (Eq.~\ref{eq:attention_w}) are row-stochastic with $\|\mathbf{W}\|_2 \leq 1$, ensuring \textbf{stable information propagation}. The $M + \binom{M}{2}$ virtual edges approximate the full adjacency $\mathbf{A}$ with only $O(M^2)$ parameters, maintaining global connectivity at low cost.

\subsection{Two-Stage Training Protocol. }
Stage~1 trains each encoder $f_{\theta_i}$ independently to minimize an $L_1$ loss on its own subgraph:
\begin{equation}
\mathcal{L}_i = \tfrac{1}{B} \sum_{b} \sum_{v \in \mathcal{V}_i} \|\mathbf{y}_{i,b,v} - f_{\theta_i}(\mathbf{X}_{i,b}, \mathbf{E}_i)_v\|_1.
\label{eq:subgraph_loss}
\end{equation}
Stage~2 freezes all encoders and trains only the virtual edges on a full-graph $L_1$ loss with $L_2$ regularization:
\begin{equation}
\mathcal{L}_{\text{ve}} = \tfrac{1}{B} \sum_{b} \sum_{v \in \mathcal{V}} \|\mathbf{y}_{b,v} - \text{IsleNet}(\mathbf{X}_b, \mathbf{E})_v\|_1 + \lambda \|\mathcal{W}_{\text{ve}}\|_2^2.
\label{eq:ve_loss}
\end{equation}
Freezing preserves subgraph gradient isolation while virtual edges learn global connectivity.

\subsection{Efficient Unlearning Protocol. }

Unlearning localizes retraining to affected subgraphs. Given $\mathcal{U}_N$, we identify:
\begin{equation}
\mathcal{S}_{\text{affected}} = \{i : \mathcal{V}_i \cap \mathcal{U}_N \neq \emptyset\}.
\label{eq:affected_subgraphs}
\end{equation}
Virtual edges connected to $\mathcal{U}_N$ are removed:
\begin{equation}
\mathcal{B}^{\text{new}} = \{b \in \mathcal{B} : \mathcal{K}_b \cap \mathcal{U}_N = \emptyset\},
\label{eq:ve_update}
\end{equation}
and rebuilt using Eq.~(\ref{eq:attention_w})--(\ref{eq:attention_mlp}). For each $i \in \mathcal{S}_{\text{affected}}$, retrain $f_{\theta_i}$ on $\mathcal{V}_i \setminus \mathcal{U}_N$:
\begin{equation}
\theta_i^{\text{new}} = \arg\min_{\theta} \mathcal{L}_i(\theta, \mathbf{X}_i', \mathbf{E}_i', \mathbf{y}_i').
\label{eq:retrain}
\end{equation}
Encoders are frozen, and virtual edges are re-optimized using Eq. (\ref{eq:ve_loss}). Finally, encoders are unfrozen to enable future fine-tuning, ensuring flexibility without compromising the unlearned state.

\begin{center}
\textbf{Algorithm 1: IsleNet Training and Unlearning}
\label{alg:islenet}
\end{center}

\begin{algorithmic}[1]
\State \textbf{Input}: $\mathcal{G} = (\mathcal{V}, \mathcal{E}, \mathbf{X})$, partitions $M$, forget nodes $\mathcal{U}_N$

\State \textbf{// Training}
\State Partition $\mathcal{G}$ into $\{\mathcal{G}_i\}_{i=1}^M$ via spectral clustering (Eq.~\ref{eq:partition})
\State If $\bar{H}_{\mathcal{P}} < 1-\epsilon$, apply balanced redistribution
\State For each $i$: train $f_{\theta_i}$ (Eq.~\ref{eq:subgraph_loss}); select $\mathcal{K}_i$ via PageRank
\State Construct $\mathcal{B}$ via attention + MLP (Eq.~\ref{eq:attention_w}--\ref{eq:attention_mlp})
\State Freeze $\{f_{\theta_i}\}$; optimize $\mathcal{B}$ (Eq.~\ref{eq:ve_loss})

\State \textbf{// Unlearning}
\State Identify $\mathcal{S}_{\text{affected}}$ (Eq.~\ref{eq:affected_subgraphs})
\State Drop $\mathcal{B}$ edges incident to $\mathcal{U}_N$
\State For each $i \in \mathcal{S}_{\text{affected}}$: retrain $f_{\theta_i}$ on $\mathcal{V}_i \setminus \mathcal{U}_N$
\State Re-initialize and re-train $\mathcal{B}^{\text{new}}$ with frozen encoders (Eq.~\ref{eq:ve_loss})

\State \textbf{Output}: Unlearned IsleNet
\end{algorithmic}

\section{Experiments.}
\label{sec:experiments}

We assess IsleNet on: 1) \textbf{accuracy parity} with a full-graph model (Scratch) before deletion; 2) \textbf{resilience after erasure} of a chosen node/edge subset; and 3) \textbf{component and efficiency analysis} via ablations plus timing and memory profiling on large graphs.

\begin{table*}[t]
\centering
\caption{Prediction Performance of Different Methods Before Unlearning (0\% Unlearning).}
\label{tab:method1}
\small
\resizebox{.95\linewidth}{!}{
\begin{tabular}{lcccccccc}
\toprule
\multirow{2}{*}{\textbf{Dataset}} & \multirow{2}{*}{\textbf{Model}} & \multirow{2}{*}{\textbf{Scratch-100\%}} & \multicolumn{4}{c}{\textbf{Baseline Methods}} & \multirow{2}{*}{\textbf{IsleNet}} \\
\cmidrule(lr){4-7}
& & & \textbf{SISA} & \textbf{STEPS} & \textbf{GraphEraser} & \textbf{GraphRevoker} & \\
\midrule
\multirow{4}{*}{\textsc{RWW}} 
& STGCN & \textcolor{gray}{0.020 $\pm$ 0.001} & 0.035 $\pm$ 0.007 & 0.082 $\pm$ 0.003 & 0.179 $\pm$ 0.060 & 0.177 $\pm$ 0.000 & \textbf{0.024 $\pm$ 0.001} \\
& ST-GAT & \textcolor{gray}{0.022 $\pm$ 0.002} & 0.035 $\pm$ 0.013 & 0.075 $\pm$ 0.004 & 0.179 $\pm$ 0.059 & 0.177 $\pm$ 0.001 & \textbf{0.025 $\pm$ 0.001} \\
& ST-GATV2 & \textcolor{gray}{0.022 $\pm$ 0.002} & 0.036 $\pm$ 0.008 & 0.085 $\pm$ 0.008 & 0.179 $\pm$ 0.059 & 0.177 $\pm$ 0.001 & \textbf{0.024 $\pm$ 0.001} \\
& ST-SAGE & \textcolor{gray}{0.022 $\pm$ 0.003} & 0.036 $\pm$ 0.010 & 0.081 $\pm$ 0.008 & 0.179 $\pm$ 0.059 & 0.178 $\pm$ 0.000 & \textbf{0.025 $\pm$ 0.001} \\
\midrule
\multirow{4}{*}{\textsc{PeMS08}} 
& STGCN & \textcolor{gray}{28.751 $\pm$ 0.117} & 34.271 $\pm$ 0.527 & 82.404 $\pm$ 9.043 & 58.994 $\pm$ 1.663 & 88.685 $\pm$ 5.865 & \textbf{30.532 $\pm$ 0.045} \\
& ST-GAT & \textcolor{gray}{28.733 $\pm$ 0.095} & 34.404 $\pm$ 0.297 & 82.244 $\pm$ 7.516 & 58.248 $\pm$ 1.175 & 90.995 $\pm$ 4.683 & \textbf{30.573 $\pm$ 0.203} \\
& ST-GATV2 & \textcolor{gray}{28.802 $\pm$ 0.023} & 34.601 $\pm$ 1.342 & 80.876 $\pm$ 10.800 & 57.938 $\pm$ 3.973 & 87.081 $\pm$ 7.951 & \textbf{30.634 $\pm$ 0.067} \\
& ST-SAGE & \textcolor{gray}{29.120 $\pm$ 0.178} & 34.133 $\pm$ 0.622 & 82.128 $\pm$ 8.982 & 64.277 $\pm$ 2.043 & 98.164 $\pm$ 0.878 & \textbf{30.394 $\pm$ 0.108} \\
\midrule
\multirow{4}{*}{\textsc{Weather}} 
& STGCN & \textcolor{gray}{3.597 $\pm$ 0.014} & 3.913 $\pm$ 0.008 & 5.449 $\pm$ 0.029 & 5.398 $\pm$ 0.214 & 5.870 $\pm$ 0.300 & \textbf{3.603 $\pm$ 0.016} \\
& ST-GAT & \textcolor{gray}{3.560 $\pm$ 0.035} & 3.902 $\pm$ 0.008 & 5.691 $\pm$ 0.089 & 4.938 $\pm$ 0.382 & 5.852 $\pm$ 0.398 & \textbf{3.654 $\pm$ 0.020} \\
& ST-GATV2 & \textcolor{gray}{3.561 $\pm$ 0.021} & 3.918 $\pm$ 0.007 & 5.557 $\pm$ 0.048 & 4.865 $\pm$ 0.232 & 6.083 $\pm$ 0.618 & \textbf{3.631 $\pm$ 0.033} \\
& ST-SAGE & \textcolor{gray}{3.572 $\pm$ 0.011} & 3.928 $\pm$ 0.011 & 5.460 $\pm$ 0.071 & 5.874 $\pm$ 0.098 & 6.034 $\pm$ 0.148 & \textbf{3.607 $\pm$ 0.030} \\
\midrule
\multirow{4}{*}{\textsc{Mobility}} 
& STGCN & \textcolor{gray}{38 102 $\pm$ 500} & 48 183 $\pm$ 1 268 & 96 095 $\pm$ 13 820 & 65 172 $\pm$ 19 092 & 129 602 $\pm$ 20 108 & \textbf{41 466 $\pm$ 3 073} \\
& ST-GAT & \textcolor{gray}{36 938 $\pm$ 402} & 47 557 $\pm$ 1 330 & 95 649 $\pm$ 10 803 & 61 125 $\pm$ 11 715 & 139 513 $\pm$ 14 559 & \textbf{42 379 $\pm$ 4 209} \\
& ST-GATV2 & \textcolor{gray}{37 346 $\pm$ 544} & 47 034 $\pm$ 1 148 & 100 220 $\pm$ 11 833 & 77 432 $\pm$ 18 665 & 136 966 $\pm$ 11 282 & \textbf{42 539 $\pm$ 3 492} \\
& ST-SAGE & \textcolor{gray}{39 068 $\pm$ 777} & 50 204 $\pm$ 1 451 & 86 902 $\pm$ 10 102 & 61 016 $\pm$ 9 939 & 125 962 $\pm$ 15 331 & \textbf{41 567 $\pm$ 4 511} \\
\bottomrule
\end{tabular}}
\end{table*}

\begin{table*}[t]
\centering
\caption{Prediction Performance of Different Methods After Unlearning (10\% Unlearning).}
\label{tab:method2}
\small
\resizebox{.95\linewidth}{!}{
\begin{tabular}{lcccccccc}
\toprule
\multirow{2}{*}{\textbf{Dataset}} & \multirow{2}{*}{\textbf{Model}} & \multirow{2}{*}{\textbf{Scratch-90\%}} & \multicolumn{4}{c}{\textbf{Baseline Methods}} & \multirow{2}{*}{\textbf{IsleNet}} \\
\cmidrule(lr){4-7}
& & & \textbf{SISA} & \textbf{STEPS} & \textbf{GraphEraser} & \textbf{GraphRevoker} & \\
\midrule
\multirow{4}{*}{\textsc{RWW}} 
& STGCN & \textcolor{gray}{0.023 $\pm$ 0.001} & 0.036 $\pm$ 0.007 & 0.095 $\pm$ 0.022 & 0.188 $\pm$ 0.067 & 0.178 $\pm$ 0.006 & \textbf{0.025 $\pm$ 0.002} \\
& ST-GAT & \textcolor{gray}{0.023 $\pm$ 0.001} & 0.038 $\pm$ 0.006 & 0.097 $\pm$ 0.025 & 0.188 $\pm$ 0.080 & 0.178 $\pm$ 0.005 & \textbf{0.026 $\pm$ 0.003} \\
& ST-GATV2 & \textcolor{gray}{0.024 $\pm$ 0.002} & 0.035 $\pm$ 0.003 & 0.090 $\pm$ 0.023 & 0.188 $\pm$ 0.081 & 0.177 $\pm$ 0.005 & \textbf{0.026 $\pm$ 0.002} \\
& ST-SAGE & \textcolor{gray}{0.023 $\pm$ 0.002} & 0.037 $\pm$ 0.011 & 0.092 $\pm$ 0.023 & 0.188 $\pm$ 0.085 & 0.178 $\pm$ 0.005 & \textbf{0.026 $\pm$ 0.003} \\
\midrule
\multirow{4}{*}{\textsc{PeMS08}} 
& STGCN & \textcolor{gray}{30.810 $\pm$ 0.147} & 34.332 $\pm$ 0.515 & 99.807 $\pm$ 12.190 & 61.315 $\pm$ 4.643 & 97.568 $\pm$ 3.789 & \textbf{31.564 $\pm$ 0.121} \\
& ST-GAT & \textcolor{gray}{30.145 $\pm$ 0.080} & 34.592 $\pm$ 0.594 & 92.950 $\pm$ 15.728 & 60.680 $\pm$ 3.484 & 91.816 $\pm$ 5.783 & \textbf{31.752 $\pm$ 0.220} \\
& ST-GATV2 & \textcolor{gray}{30.054 $\pm$ 0.143} & 33.724 $\pm$ 0.271 & 91.348 $\pm$ 17.671 & 59.433 $\pm$ 1.374 & 91.973 $\pm$ 8.148 & \textbf{32.011 $\pm$ 0.135} \\
& ST-SAGE & \textcolor{gray}{30.304 $\pm$ 0.327} & 35.259 $\pm$ 0.517 & 94.038 $\pm$ 13.147 & 59.925 $\pm$ 1.202 & 96.225 $\pm$ 1.806 & \textbf{31.953 $\pm$ 0.198} \\
\midrule
\multirow{4}{*}{\textsc{Weather}} 
& STGCN & \textcolor{gray}{3.581 $\pm$ 0.020} & 3.956 $\pm$ 0.011 & 5.480 $\pm$ 0.061 & 5.816 $\pm$ 0.089 & 5.989 $\pm$ 0.380 & \textbf{3.703 $\pm$ 0.013} \\
& ST-GAT & \textcolor{gray}{3.590 $\pm$ 0.002} & 3.919 $\pm$ 0.009 & 5.475 $\pm$ 0.114 & 5.153 $\pm$ 0.491 & 5.944 $\pm$ 0.365 & \textbf{3.751 $\pm$ 0.062} \\
& ST-GATV2 & \textcolor{gray}{3.569 $\pm$ 0.009} & 3.975 $\pm$ 0.010 & 5.766 $\pm$ 0.027 & 5.016 $\pm$ 0.632 & 5.545 $\pm$ 0.653 & \textbf{3.801 $\pm$ 0.034} \\
& ST-SAGE & \textcolor{gray}{3.584 $\pm$ 0.005} & 3.996 $\pm$ 0.020 & 5.520 $\pm$ 0.166 & 5.399 $\pm$ 0.264 & 6.312 $\pm$ 0.499 & \textbf{3.811 $\pm$ 0.054} \\
\midrule
\multirow{4}{*}{\textsc{Mobility}} 
& STGCN & \textcolor{gray}{38 602 $\pm$ 758} & 48 938 $\pm$ 1 039 & 100 059 $\pm$ 16 828 & 73 745 $\pm$ 17 019 & 131 529 $\pm$ 14 613 & \textbf{44 298 $\pm$ 5 420} \\
& ST-GAT & \textcolor{gray}{37 815 $\pm$ 806} & 47 807 $\pm$ 1 297 & 102 763 $\pm$ 13 037 & 65 775 $\pm$ 14 700 & 124 914 $\pm$ 15 670 & \textbf{44 383 $\pm$ 5 362} \\
& ST-GATV2 & \textcolor{gray}{37 472 $\pm$ 741} & 49 129 $\pm$ 1 285 & 94 374 $\pm$ 12 208 & 76 865 $\pm$ 16 989 & 128 456 $\pm$ 18 644 & \textbf{45 944 $\pm$ 6 343} \\
& ST-SAGE & \textcolor{gray}{39 066 $\pm$ 596} & 50 254 $\pm$ 1 770 & 89 163 $\pm$ 10 121 & 60 593 $\pm$ 10 043 & 122 181 $\pm$ 15 292 & \textbf{43 920 $\pm$ 4 706} \\
\bottomrule
\end{tabular}}
\end{table*}

\subsection{Experimental Setup.}\leavevmode

\noindent\textbf{Dataset:} To evaluate the scalability of our method, we selected spatiotemporal graph data spanning a range of sizes, with up to 3,220 nodes. These datasets include: RWW~\cite{guo2024hydronet}, a 23-node network representing water depth in a sewage system; PeMS08~\cite{j49q-ch56-25}, a 170-node traffic flow network in California; Global Weather~\cite{noaa_psl_2025}, a 1,000-node global daily temperature network; and Human Mobility Flow~\cite{kang2020multiscale}, a 3,220-node mobility network capturing daily population movement. The datasets consist of time series ranging from 3,000 to 18,000 time steps, making them large-scale. We split the data temporally into training (70\%), validation (15\%), and test (15\%) sets.

\noindent\textbf{Baselines:} We compare our approach against several state-of-the-art baselines: Scratch (full graph training with no unlearning), SISA~\cite{bourtoule2021machine}, STEPS~\cite{guo2025efficient}, GraphEraser~\cite{chen2022grapheraser}, and GraphRevoker~\cite{zhang2025dynamic}, evaluated on four spatiotemporal graph backbones: STGCN, ST-SAGE, ST-GAT, and ST-GATV2. We fix the number of subgraphs $M=4$. We note that GraphEraser and GraphRevoker were originally designed for recommender systems with static or discrete-time dynamic graphs; we include them as general-purpose graph-unlearning baselines to provide context rather than as ST-specific competitors.

\noindent\textbf{Metrics:} We record evaluation metrics including MAE, MSE, RMSE, Trend F1, and $R^2$. MAE are reported in the Results section on the original scale, with mean and standard deviation. Runtime, memory, and CPU costs are also measured.

\noindent\textbf{Fair and Robust Setup:} Model parameters are tuned to reach $R^2 > 0.9$ on RWW, PeMS08, and Mobility; Weather caps at $R^2 \approx 0.67$ due to high per-station stochasticity, so we report MAE as the primary metric on Weather. To control capacity differences across partitions, the hidden dimension of each subgraph encoder is scaled to the unlearning proportion so that the non-unlearned model matches full-graph $R^2$. We evaluate at 10\% unlearning, a large proportion relative to typical real-world requests, with 5 fixed random seeds for reproducibility.

\noindent\textbf{Implementation Details:} For partitioning, we use spectral clustering with $M=4$ subgraphs and normalized entropy threshold $\epsilon = 0.1$ (Eq.~\ref{eq:partition}). The PageRank damping factor is $\alpha = 0.85$. For virtual edges, the residual scaling factor is $\beta = 0.1$ (Eq.~\ref{eq:ve_enhance}) and the regularization weight is $\lambda = 10^{-4}$ (Eq.~\ref{eq:ve_loss}). Stage~1 uses Adam with learning rate $10^{-3}$; Stage~2 uses Adam with learning rate $5 \times 10^{-4}$. Early stopping patience is 10 epochs for both stages. Batch size is 32. For datasets without explicit node coordinates, we derive two-dimensional topological coordinates via force-directed graph layout on the adjacency matrix~\cite{fruchterman1991graph}; spatial entropy is then computed on a $10 \times 10$ grid over these coordinates. All experiments are conducted on an NVIDIA A100 GPU with 5 fixed random seeds.

\subsection{Results. }

As shown in Table~\ref{tab:method1}, at 0\% unlearning IsleNet closely matches Scratch-100\%: on PeMS08 with STGCN its MAE is 30.532 $\pm$ 0.045, only ${\sim}6.2\%$ above Scratch (28.751 $\pm$ 0.117), i.e., ${\sim}94\%$ of full-graph performance, and on Weather it is nearly identical (3.603 $\pm$ 0.016 vs.\ 3.597 $\pm$ 0.014). Baselines degrade sharply: GraphEraser and GraphRevoker, designed for recommender systems, reach MAEs up to ${\sim}208\%$ worse than Scratch on PeMS08 (e.g., 88.685 $\pm$ 5.865 for GraphRevoker); STEPS, using uniform partitioning and weighted averaging, is adequate only on Weather (5.449 vs.\ 3.597) and fails on PeMS08 and Mobility (82.404 $\pm$ 9.043 on PeMS08); SISA, relying on overlapping partitions, is the strongest baseline (34.271 $\pm$ 0.527 on PeMS08) yet remains clearly inferior to IsleNet.

At 10\% unlearning (Table~\ref{tab:method2}), simulating extensive concurrent requests, all methods degrade, but IsleNet stays comparable to Scratch-90\%: on PeMS08 with STGCN, 31.564 $\pm$ 0.121 vs.\ 30.810 $\pm$ 0.147 (${\sim}2.4\%$ gap), with similarly modest increases elsewhere (${\sim}8.7\%$ on RWW: 0.025 vs.\ 0.023; ${\sim}3.4\%$ on Weather: 3.703 vs.\ 3.581). Baselines suffer far larger losses on PeMS08: STEPS (99.807 $\pm$ 12.190) and GraphRevoker (97.568 $\pm$ 3.789) degrade extremely, while GraphEraser (61.315 $\pm$ 4.643) and SISA (34.332 $\pm$ 0.515) perform better but remain inferior.

To summarize proximity to the full-graph baseline, we compute two measures between each method's MAE and Scratch's MAE: Bounded Similarity $= \max(0, 1 - |\mathrm{MAE}_m - \mathrm{MAE}_s|/\mathrm{MAE}_s)$ and Ratio Score $= \min(\mathrm{MAE}_m,\mathrm{MAE}_s)/\max(\mathrm{MAE}_m,\mathrm{MAE}_s)$, averaged across all datasets, backbones, and unlearning rates. Table~\ref{tab:similarity_summary_vertical} confirms that IsleNet (95.4\%/95.7\%) is closest to Scratch; SISA follows at 86.2\%/89.5\%; other baselines fall below 72\% on both measures.

\begin{table}[t]
\centering
\caption{Average similarity to Scratch models (higher is better). Computed across 4 datasets, 4 backbones, and 2 unlearning rates (0\% and 10\%).}
\label{tab:similarity_summary_vertical}
\small
\setlength{\tabcolsep}{4pt}
\begin{tabular}{lcc}
\toprule
\textbf{Method} & \textbf{Bounded Sim (\%)} & \textbf{Ratio Score (\%)} \\
\midrule
SISA         & 86.22 & 89.51 \\
STEPS        & 50.81 & 67.15 \\
GraphEraser  & 58.26 & 71.35 \\
GraphRevoker & 52.39 & 65.16 \\
IsleNet      & \textbf{95.35} & \textbf{95.65} \\
\bottomrule
\end{tabular}
\end{table}

\subsection{Ablation Study. }
\label{sec:exp:ablation}

We ablate IsleNet's four components on PeMS08 with STGCN (Table~\ref{tab:ablation}). Replacing entropy-verified spectral partitioning with random partitioning raises MAE by 49.5\%/58.2\% (0\%/10\% unlearning), confirming that locally coherent, balanced subgraphs are essential. Removing virtual edges costs 15.7\%/26.2\%, showing global context via attention-based key-node communication matters. Replacing two-stage training with joint training raises MAE by over 25\% due to gradient interference between structural modules. Replacing localized unlearning with full retraining costs $\sim$22\% MAE and $\sim$10$\times$ throughput. Together these confirm the combined value of entropy-verified partitioning, virtual edges, staged optimization, and localized unlearning.

\begin{table}[h]
\centering
\caption{Ablation study on PeMS08 with STGCN.}
\label{tab:ablation}
\footnotesize
\begin{tabular}{@{}lcc@{}}
\toprule
\textbf{Configuration} & \textbf{0\% Unl. MAE} & \textbf{10\% Unl. MAE} \\
\midrule
\multicolumn{3}{l}{\textit{Baselines}} \\
Full Graph (Scratch) & 28.751 & -- \\
Retrained Graph (90\%) & -- & 30.810 \\
\midrule
\multicolumn{3}{l}{\textit{IsleNet Variants}} \\
Full IsleNet & 30.532 & 31.564 \\
Random Part. & 45.635 & 49.938 \\
No Virtual Edges & 35.322 & 39.840 \\
No Two-Stage Tr. & 38.261 & 39.551 \\
No Localized Unl. & 37.295 & 37.764 \\
\bottomrule
\end{tabular}
\end{table}

\subsection{Efficiency and Capacity. }
We evaluate scalability on PeMS08 (STGCN backbone, 10\% node unlearning, 5 seeds). The monolithic ST-GNN needs ${\sim}600$s to train and ${\sim}60$s to unlearn 17 nodes. IsleNet with $M{=}4$ cuts training to ${\sim}360$s and unlearning to ${\sim}5$s (a ${\sim}12\times$ wall-clock speedup) by retraining only affected subgraphs (Eq.~\ref{eq:affected_subgraphs}--\ref{eq:retrain}). At $M{=}8$, total training stays ${\sim}360$s as the growing $\binom{M}{2}$ virtual-edge cost (Eq.~\ref{eq:attention_w}) offsets smaller per-subgraph cost; unlearning remains ${\sim}4.5$s. At $M{=}16$, training is ${\sim}380$s and unlearning drops to ${\sim}4$s with ${\sim}10$-node subgraphs. IsleNet is thus especially efficient under frequent unlearning (see also Section~\ref{sec:exp:ablation}).

\section{Conclusion.}
\label{sec:conclusion}
We introduced IsleNet, a spatial-entropy-verified partitioning framework for spatiotemporal graph unlearning. Its two components (entropy-verified spectral partitioning and decoupled virtual edges) together enable exact localized unlearning with minimal accuracy loss. Across four real-world benchmarks IsleNet retains up to 94\% of full-graph accuracy while cutting unlearning time by an order of magnitude. Validation on industry-scale graphs with millions of nodes and a dedicated ablation isolating entropy verification from spectral clustering are natural next steps.

\section*{Acknowledgment.}
This work was partially supported by NSF Awards No. 2318641 and No. 2112631.

\bibliographystyle{siamplain}
\bibliography{references}

\end{document}